\def\isarxiv{1} 

\ifdefined\isarxiv
\documentclass[11pt]{article}

\else

\documentclass[twoside]{article}
\usepackage[authoryear]{natbib}
\usepackage{aistats2024}

\fi

\usepackage{amsmath}
\usepackage{amsthm}
\usepackage{amssymb}
\usepackage{algorithm}
\usepackage{subfig}
\usepackage{algpseudocode}
\usepackage{graphicx}
\usepackage{grffile}
\usepackage{wrapfig,epsfig}
\usepackage{url}
\usepackage{xcolor}
\usepackage{epstopdf}

\usepackage{bbm}
\usepackage{dsfont}
\usepackage[T1]{fontenc}

\allowdisplaybreaks

\ifdefined\isarxiv

\usepackage{tikz}
\usepackage{hyperref}  
\hypersetup{colorlinks=true,citecolor=blue,linkcolor=blue} 
\usetikzlibrary{arrows}
\usepackage[margin=1in]{geometry}

\else

\usepackage{microtype}
\usepackage{hyperref}
\definecolor{mydarkblue}{rgb}{0,0.08,0.45}
\hypersetup{colorlinks=true, citecolor=mydarkblue,linkcolor=mydarkblue}

\fi

\newtheorem{theorem}{Theorem}[section]
\newtheorem{lemma}[theorem]{Lemma}
\newtheorem{definition}[theorem]{Definition}

\newtheorem{corollary}[theorem]{Corollary}

\newtheorem{assumption}[theorem]{Assumption}

\newcommand{\wh}{\widehat}

\newcommand{\R}{\mathbb{R}}

\renewcommand{\d}{\mathrm{d}}

\renewcommand{\hat}{\wh}

\DeclareMathOperator*{\E}{{\mathbb{E}}}

\makeatletter
\newcommand*{\RN}[1]{\expandafter\@slowromancap\romannumeral #1@}
\makeatother

\usepackage{lineno}

\begin{document}

\ifdefined\isarxiv

\date{}

\title{An Automatic Learning Rate Schedule Algorithm for Achieving Faster Convergence and Steeper Descent}
\author{
Zhao Song\thanks{\texttt{zsong@adobe.com}. Adobe Research.}
\and 
Chiwun Yang\thanks{\texttt{christiannyang37@gmail.com}. 
Sun Yat-sen University.}
}

\else

\twocolumn[
\aistatstitle{An Automatic Learning Rate Schedule Algorithm for Achieving Faster Convergence and Steeper Descent}
\aistatsauthor{Zhao Song \And Chiwun Yang}
\aistatsaddress{Adobe Research \And ???}
]

\fi

\ifdefined\isarxiv
\begin{titlepage}
  \maketitle
  \begin{abstract}
    The delta-bar-delta algorithm is recognized as a learning rate adaptation technique that enhances the convergence speed of the training process in optimization by dynamically scheduling the learning rate based on the difference between the current and previous weight updates. While this algorithm has demonstrated strong competitiveness in full data optimization when compared to other state-of-the-art algorithms like Adam and SGD, it may encounter convergence issues in mini-batch optimization scenarios due to the presence of noisy gradients.

In this study, we thoroughly investigate the convergence behavior of the delta-bar-delta algorithm in real-world neural network optimization. To address any potential convergence challenges, we propose a novel approach called RDBD (Regrettable Delta-Bar-Delta). Our approach allows for prompt correction of biased learning rate adjustments and ensures the convergence of the optimization process. Furthermore, we demonstrate that RDBD can be seamlessly integrated with any optimization algorithm and significantly improve the convergence speed.

By conducting extensive experiments and evaluations, we validate the effectiveness and efficiency of our proposed RDBD approach. The results showcase its capability to overcome convergence issues in mini-batch optimization and its potential to enhance the convergence speed of various optimization algorithms. This research contributes to the advancement of optimization techniques in neural network training, providing practitioners with a reliable automatic learning rate scheduler for achieving faster convergence and improved optimization outcomes.
  \end{abstract}
  \thispagestyle{empty}
\end{titlepage}

{
}
\newpage

\else

\begin{abstract}
The delta-bar-delta algorithm is recognized as a learning rate adaptation technique that enhances the convergence speed of the training process in optimization by dynamically scheduling the learning rate based on the difference between the current and previous weight updates. While this algorithm has demonstrated strong competitiveness in full data optimization when compared to other state-of-the-art algorithms like Adam and SGD, it may encounter convergence issues in mini-batch optimization scenarios due to the presence of noisy gradients.

In this study, we thoroughly investigate the convergence behavior of the delta-bar-delta algorithm in real-world neural network optimization. To address any potential convergence challenges, we propose a novel approach called RDBD (Regrettable Delta-Bar-Delta). Our approach allows for prompt correction of biased learning rate adjustments and ensures the convergence of the optimization process. Furthermore, we demonstrate that RDBD can be seamlessly integrated with any optimization algorithm and significantly improve the convergence speed.

By conducting extensive experiments and evaluations, we validate the effectiveness and efficiency of our proposed RDBD approach. The results showcase its capability to overcome convergence issues in mini-batch optimization and its potential to enhance the convergence speed of various optimization algorithms. This research contributes to the advancement of optimization techniques in neural network training, providing practitioners with a reliable automatic learning rate scheduler for achieving faster convergence and improved optimization outcomes.
\end{abstract}

\fi

\section{Introduction}


\begin{figure}[!ht] 
\centering
    \includegraphics[scale=0.45]{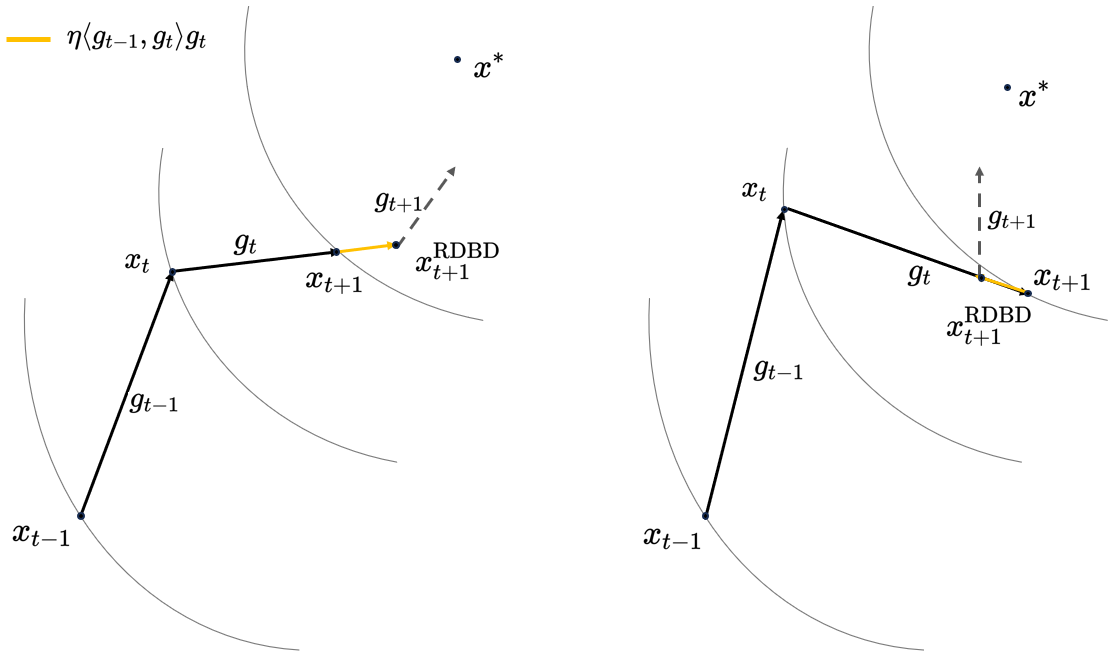}
    \caption{A 2D simulation of RDBD algorithm optimization in step $t$. }
    \label{fig:simulation}
\end{figure}

The learning rate is a crucial hyper-parameter in neural network training processes, as it significantly influences the trained model's convergence speed and overall performance. Numerous studies have demonstrated that inappropriate choice and scheduler of learning rate can hinder convergence during training and adversely impact the model's performance \cite{fah88, jac88, bl88, szl13, zei12, skyl17, ylwj19, gas19, lwm19, wlb+19, jsf+20, lew21}. It is imperative to carefully select an optimal learning rate, and develop an appropriate scheduler to ensure successful training and maximize the model's capabilities. 

Meanwhile, the delta-bar-delta algorithm, which has its roots in the works of \cite{wh60, bs81, jac88, mw90, sut92, ois05}, schedules the learning rate by assigning a unique learning rate to each weight in the model and dynamically updating these learning rates as individual optimization processes. It also has been recognized as a meta-learning algorithm \cite{vd02, van18, hams21, fal17} since it optimizes the learning rate itself, effectively {\it learning to learn} during the training process. In a nutshell, the delta-bar-delta algorithm can be summarized as follows: for each weight in the model, the learning rate is adjusted based on the dot product of the current gradient and the previous gradient. If the dot product is positive, the learning rate increases, and if it is negative, the learning rate decreases. This adaptive rule for learning rate greatly influenced subsequent advancements in optimization algorithms, including AdaGrad \cite{dhs11}, Adadelta \cite{zei12}, RMSProp \cite{hss12}, Adam \cite{kb14}, Adamax \cite{lih17}, and Nadam \cite{doz16}.

Inferring this learning rate adaptation method only requires a simple derivation. Denote the loss function $f(x)$ of a parameter weight $x$ in back-propagation. In this algorithm, we consider the loss function $f(x)$ of a parameter weight $x$ in the back-propagation process. Begins by initializing a learning rate $\alpha$. At each optimization step $t$, we update the parameters $x_t$ using the equation $x_t = x_{t-1} - \alpha_{t-1} \nabla f(x_{t-1})$, where $\nabla f(x_{t-1})$ represents the gradient of the loss function with respect to the parameters at the previous step. Our objective is to minimize $f(x_t)$ while keeping $x_{t-1}$ fixed. To achieve this with minimal regret, we optimize the parameter $\alpha_{t-1}$ using the following equation:
\begin{align*}
\alpha_{t} = & ~ \alpha_{t-1} - \eta \nabla_{\alpha_{t-1}} f(x_t) \\
= & ~ \alpha_{t-1} + \eta \langle \nabla f(x_{t-1}), \nabla f(x_t) \rangle
\end{align*}
where $\eta$ represents the learning rate for updating $\alpha$. Furthermore, we provide a simple theorem to show the convergence of the delta-bar-delta algorithm in full data optimization, which showcases its strong competitiveness:
\begin{theorem}[Convergence guarantee for delta-bar-delta algorithm in full batch optimization, informal version of Theorem~\ref{thm:convergence_dbd:formal}]\label{thm:convergence_dbd:informal}
    Suppose $\nabla f(x)$ is Lipschitz-smooth with constant $L$ $\| \nabla f(x) - \nabla f(y) \|_2 \leq L \| x - y \|, \forall x, y \in \R^d$ and gradient of loss function $\nabla f(x)$ is $\sigma$-bounded that $\| \nabla f(x ) \|_2 \leq \sigma, \forall x \in \R^d$. We initialize $\alpha_0 = \frac{1}{L}$ and $\eta = \frac{\gamma}{T\sigma^2L}$, where $\gamma \in [0, 1)$ is denoted as a scalar. Denote $f^* := \min_{x \in \R^d} f(x)$. Let $\epsilon > 0$ be denoted as the error of training. We run the delta-bar-delta algorithm at most $\frac{2 L( f(x_0) - f^* )}{(1 - \gamma^2) \epsilon^2}$ times iterations, we have
    \begin{align*}
        \min_t \{ \| \nabla f(x_t) \|_2 \} \leq \epsilon
    \end{align*}
\end{theorem}
For a detailed derivation process and a simple convergence proof of the delta-bar-delta algorithm, please refer to Appendix~\ref{sec:convergence_dbd}.

However, the presence of noisy gradients in mini-batch optimization poses a challenge when using the delta-bar-delta algorithm for real-world neural network optimization. Biased or noisy gradients can adversely affect learning rate updates, potentially leading to erroneous adjustments and the collapse of the training process \cite{qk20, sed20, wlg+19, zlm21, zlsu21, jka+17}. To address this issue, prior works have explored regularization techniques like sign function and exponential function to mitigate the impact of noisy gradients \cite{jac88, mw90, sut92, agks19, zhsj19, lxls19, bcr+17, ois05, aafw22}. These strategies aim to stabilize the learning process and improve the algorithm's resilience to noisy gradient estimates. Handling noisy gradients effectively remains an ongoing focus of research in neural network optimization.

In this paper, we introduce a novel approach called the Regrettable Delta-Bar-Delta (RDBD) algorithm, which aims to improve loss reduction and convergence in mini-batch optimization. {\bf We define the term {\it regrettable} as the capacity to reevaluate and amend the previous learning rate update. This capability allows for a more nuanced and adaptable approach to optimizing the learning process.} The RDBD algorithm incorporates a buffering mechanism for the learning rate updates, allowing for verification of the effectiveness of the previous update when a new gradient is obtained. Specifically, at each step $t+1$, we define $h_{t+1} := \langle g_{t-1}, g_{t} \rangle$, where $g_{t}$ represents the gradient at step $t$. Similarly, we have $h_{t} = \langle g_{t}, g_{t-1} \rangle$. If the product $h_{t+1} \cdot h_{t}$ is negative, it indicates that the previous learning rate update $\alpha_t = \alpha_{t-1} + \eta h_{t-1}$ is biased. In such cases, we revert this update by implementing $\alpha_t \leftarrow \alpha_t - \eta h_{t-1}$, making the learning rate adjustment regrettable  (Algorithm ~\ref{alg:RDBD}).

We state our main contributions and results in the following section.

\subsection{Contributions and Results}

In this study, our contributions include the following:
\begin{itemize}
    \item We propose the Regrettable Delta-Bar-Delta (RDBD) algorithm, which ensures the convergence of the delta-bar-delta algorithm in mini-batch optimization. RDBD serves as a learning rate schedule method that can be combined with any optimizing algorithm, directly enhancing the convergence speed.
    \item We provide theoretical proofs demonstrating that our RDBD algorithm exhibits steeper descent regardless of the optimization algorithm it is combined with. Additionally, we establish the convergence guarantee for the RDBD algorithm in mini-batch optimization.
    \item We conduct experiments to evaluate the efficacy of RDBD by applying it to SGD and Adam on popular datasets such as Cifar-10 and MNIST. The experimental results demonstrate a significant improvement, highlighting the strong capability of our RDBD method as a learning rate schedule algorithm.
\end{itemize}

Then, we show our theorem of the main result as follows:

\begin{theorem}[Main result, formal version of Theorem~\ref{thm:main_result:formal}]\label{thm:main_result:informal}
    For a loss function $f: \R^d \rightarrow \R$ and the weight $x$ of a vector parameter of a model. We run RDBD algorithm on it. We first initialize $x_0 = x$, $\alpha_0 = \frac{\sqrt{f(x_0) - f^*}}{\sigma\sqrt{L T}}$, $\eta = \frac{\gamma\sqrt{f(x_0) - f^*}}{T\sigma^3\sqrt{L T}}$ where $\gamma \in (0, 1)$. Denote $f^* := \min_{x \in \R^d} f(x)$. Let $\epsilon > 0$ be denoted as the error of training. We have
    \begin{itemize}
        \item {\bf Steeper Loss Descent.} At step $t$, the RDBD algorithm accelerates loss reduction as follows:
        \begin{align*}
            f(x_{t+1}^{\rm RDBD}) \leq f(x_{t+1}) \leq f(x_t)
        \end{align*}
        \item {\bf Convergence. } Let
        \begin{align*}
            T = \frac{1}{\epsilon^2} \cdot \sigma \sqrt{L(f(x_0) - f^*)} (\frac{1}{1- \gamma} + \frac{1}{2} ( 1 + \gamma )) 
        \end{align*}
        then, at most $T$ time iterations RDBD algorithm, we have
        \begin{align*}
            \min_t \{ \| \nabla f(x_t) \|_2 \} \leq \epsilon
        \end{align*}
    \end{itemize}
\end{theorem}

Please see Appendix~\ref{sub:main_result_proof} for the proof of Theorem~\ref{thm:main_result:informal}.

Furthermore, we present the main experimental result of our paper (Figure~\ref{fig:main_result}), where we run four algorithms including Adam, Adam+RDBD (a combination of Adam and RDBD), SGD, RDBD on two datasets: MNIST \cite{lcb09} and Cifar-10 \cite{kh09}. Please refer to Section~\ref{sec:experiment} for more experimental details.

\begin{figure*}
    \centering
    \includegraphics[width=\textwidth]{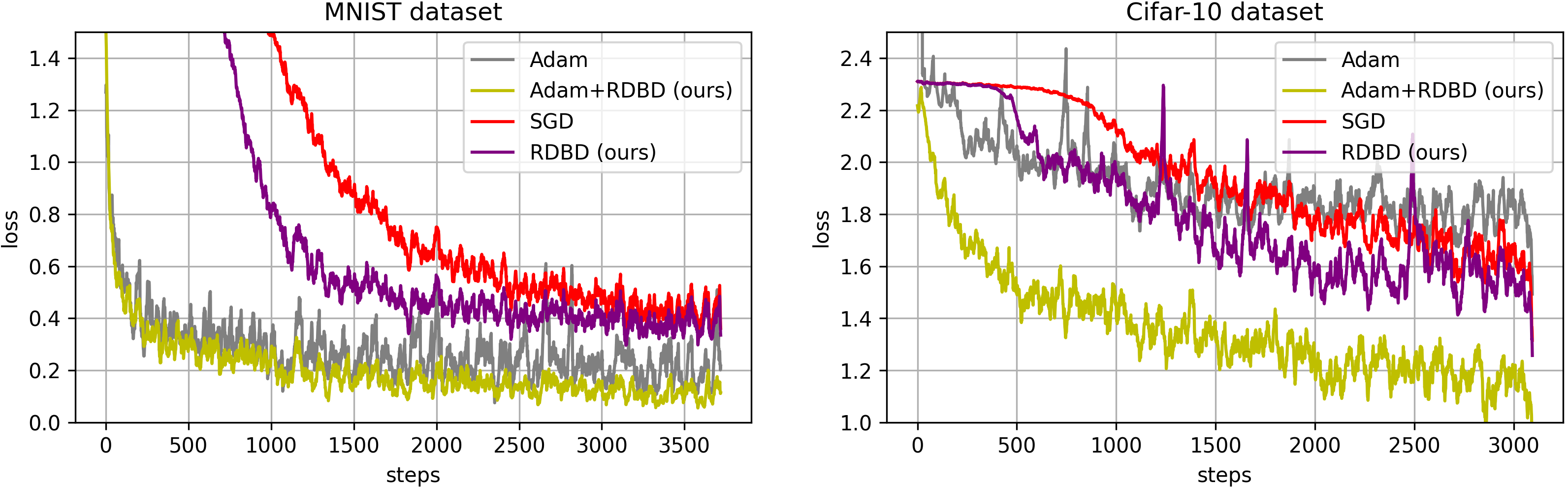}
    \caption{Comparison of Adam, Adam+RDBD, SGD, RDBD algorithms on MNIST dataset and Cifar-10 dataset.}
    \label{fig:main_result}
\end{figure*}

\section{Related Work}

In this section, we review related research on Learning Rate Adaptation Methods, Meta-learning Algorithms, and Optimization and Convergence of Deep Neural Network. These topics are closely connected to our work.

\paragraph{Learning Rate Adaptation Methods.} Previous research explored the correlation between the scheduled learning rate and convergence speed in neural network optimization and they developed a range of learning rate adaptation algorithms to enhance convergence speed \cite{awzd17, la19, gkkn19, din21, dhs11, zei12, hss12, kb14, lih17, doz16, gch+19, jagg20, zzl+18, bac+17, all17, lo19, bcr+17, dm23, lavb23, ysm23, ljh+19, jzz+21, rc21}. \cite{gkkn19} investigates the behavior and convergence rates of the final iterate in stochastic convex optimization problems, focusing on streaming least squares regression with and without strong convexity. The study reveals that polynomially decaying learning rate schemes lead to highly sub-optimal performance compared to the statistical minimax rate, regardless of whether the time horizon is known in advance. However, it is shown that step decay schedules, which geometrically reduce the learning rate, offer significant improvements and achieve rates close to the statistical minimax rate with only logarithmic factors. In the study of \cite{ljh+19}, they explore the mechanism behind warmup and discover a problem with the adaptive learning rate, specifically its large variance in the early stage. They propose that warmup acts as a variance reduction technique and provide empirical and theoretical evidence to support this hypothesis. Additionally, they introduce RAdam, a new variant of Adam that rectifies the variance of the adaptive learning rate. D-Adaptation (\cite{dm23}) is a novel approach that automatically sets the learning rate to achieve optimal convergence for minimizing convex Lipschitz functions. It eliminates the need for back-tracking or line searches and does not require additional function value or gradient evaluations per step.

\paragraph{Meta-learning Algorithms.} Meta-learning, also known as {\it learning to learn}, is a subfield of machine learning that focuses on improving the performance of learning algorithms by changing certain aspects of the learning process based on the results of previous learning experiences. In meta-learning, the goal is to develop algorithms or models that can learn how to learn. Instead of focusing solely on solving a specific task, meta-learning algorithms aim to acquire knowledge and strategies that can be applied to a wide range of related tasks \cite{vd02, van18, hams21, fal17, sd03, hyc01, bvl23, hcw+23, csy23b, cjn23, wfz+23, sl23}. \cite{wfz+23} introduces the Adaptive Compositional Continual Meta-Learning (ACML) algorithm, which addresses the challenge of handling heterogeneous and sequentially available few-shot tasks in continual meta-learning. ACML overcomes these limitations by employing a compositional premise that associates a task with a subset of mixture components, enabling meta-knowledge sharing across heterogeneous tasks. To effectively utilize minimal annotated examples in new languages for few-shot cross-lingual semantic parsing, \cite{sl23} introduces a first-order meta-learning algorithm. This algorithm trains the semantic parser using high-resource languages and optimizes it for cross-lingual generalization to lower-resource languages.

\paragraph{Optimization and Convergence of Deep Neural Network. }  Extensive prior research has played a pivotal role in elucidating the optimization and convergence principles underlying deep neural networks, thereby explaining their remarkable performance across diverse tasks. Moreover, these studies have significantly contributed to the enhancement of efficiency in deep learning frameworks. Notable works in this domain include those by authors such as \cite{zsj+17, zsd17, ll18, dzps18, azls19a, azls19b, adh+19a, adh+19b, sy19, cgh+19, zmg19, cg19, zg19, os20, jt19, lss+20, hlsy21, zpd+20, bpsw21,  zkv+20, syz21, szz21, als+22, mosw22, zha22, gms23, dls23, gsy23_coin, lsz23, wyw+23, qsy23, csy23a, clh+23, gsx23, rkk19, swy23, jsps16, gswy23, sy23, lsy23, lsx+23}. In \cite{lsy23}, they propose a generalized form of Federated Adversarial Learning (FAL) based on centralized adversarial learning. FAL incorporates an inner loop on the client side for generating adversarial samples and an outer loop for updating local model parameters. On the server side, FAL aggregates local model updates and broadcasts the aggregated model. The authors introduce a global robust training loss and formulate FAL training as a min-max optimization problem. In \cite{gms23} they define a neural function using an exponential activation function and apply the gradient descent algorithm to find optimal weights. \cite{rhs+16} focuses on nonconvex finite-sum problems and investigates the application of stochastic variance reduced gradient (SVRG) methods. While SVRG and related methods have gained attention for convex optimization, their theoretical analysis has primarily focused on convexity. In contrast, this research establishes non-asymptotic convergence rates of SVRG for nonconvex optimization and demonstrates its superiority over stochastic gradient descent (SGD) and gradient descent. The analysis also reveals that SVRG achieves linear convergence to the global optimum in a specific subclass of nonconvex problems. Additionally, the study explores mini-batch variants of SVRG, demonstrating theoretical linear speedup through minibatching in parallel settings.

\section{Preliminary}

In this section, we provide the notations and formally define the problem that we aim to address in this paper. 

\subsection{Notations}

In this paper, we adopt the following notations: The number of dimensions of a neural network is denoted by $d$. The set of real numbers is represented by $\R$. A vector with $d$ elements, where each entry is a real number, is denoted as $x \in \mathbb{R}^{d}$. The weight vector parameter in a neural network is denoted as $x$. We use $x$ to denote the weight of a vector parameter in a neural network. For any positive integer $n$, we use $[n]$ to denote $\{1,2,\cdots, n\}$. The $\ell_p$ norm of a vector $x$ is denoted as $\| x \|_p$, for examples, $\| x \|_1 := \sum^n_{i=1} | x_i |$, $\| x \|_2 := ( \sum^n_{i=1} x_i^2 )^{1/2}$ and $\| x \|_\infty := \max_{i \in [n]} | x_i |$. For two vectors $x, y \in \R^d$, we denote $\langle x, y \rangle = \sum^n_{i=1}$ for $i \in [d]$. The loss function of a weight vector parameter $x \in \mathbb{R}^d$ on the entire dataset is denoted as $f: \mathbb{R}^d \rightarrow \mathbb{R}$. Specifically, $f_\xi(x)$ represents the loss function of $x$ on the batch data $\xi$. The learning rate during training is denoted as $\alpha$. In particular, $\eta$ is used to represent the learning rate for updating $\alpha$ in the delta-bar-delta algorithm. We consider the optimization process as consisting of $T$ update steps. For an integer $t \in [T]$, at step $t$, we have the weight vector parameter $x_t$ and the learning rate $\alpha_t$. We use $\nabla f(x)$ to denote $\frac{\partial f(x)}{\partial x}$. We use $\E[]$ to denote expectation.

\subsection{Problem Setup}

While the delta-bar-delta algorithm has demonstrated impressive effectiveness in learning rate scheduling, challenges related to noisy gradient updates and convergence in mini-batch optimization persist. In our research, we specifically focus on addressing these challenges in relation to the weight of a parameter $x$ within a neural network. It is important to note that in this context, $x$ does not represent the weight of the entire neural network, but rather the weight of a specific vector parameter within it. The weights of a neural network can be divided into multiple vectors, each with its own associated learning rate. Accordingly, our proposed method is applied separately to each individual $x$, considering the specific vector it belongs to.

We consider the problem of unconstrained nonconvex minimization of vector $x \in \R^d$:
\begin{align*}
    \min_{x \in \R^d} f(x)
\end{align*}
where $f: \R^d \rightarrow \R$ is the unbiased loss of vector $x$ in back-propagation. Our goal is to achieve loss descent and accelerate the rate of loss reduction through the use of learning rate schedule algorithms. Additionally, we aim to provide a convergence guarantee for the optimization process.

Now let's first formulate the assumptions of problem in this paper. The conditions assumed for this problem are as follows:
\begin{itemize}
    \item {\bf Lipschitz-continuity.} The gradient of $f(x)$ (denoted as $\nabla f(x)$) is Lipschitz-smooth with constant $L$, that is $\| \nabla f(x) - \nabla f(y) \|_2 \leq L \| x - y \|_2, \forall x, y \in \R^d$.
    \item {\bf Bounded update (gradient).} Denote the weight update of $x$ at step $t$ as $g_t$, $g_t$ is $\sigma$-bounded that $\| g_t \|_2 \leq \sigma$, for $t \in [T]$.
    \item {\bf Unbiased gradient estimator.} For $t \in [T]$, $g_t$ satisfies that:
    \begin{itemize}
        \item $\E[g_t] = \nabla f(x_t)$, where $x_t$ is the weight at step $t$.
        \item $\langle \nabla f(x_t), g_t \rangle \geq \mu \| g_t \|_2^2$.
    \end{itemize}
\end{itemize}
After that, we provide our main objectives addressed in this paper:
\begin{itemize}
    \item {\bf Loss descent.} By adjusting the learning rate, we aim to achieve a decrease such that
    \begin{align*}
        \E[f(x_{t+1})] \leq f(x_t)
    \end{align*}
    \item {\bf Steeper loss descent. }  We seek to further accelerate the rate of loss reduction by utilizing learning rate schedule algorithms. Specifically, we introduce the concept of learning-rate-scheduled loss $f(x_{t+1}')$, which satisfies
    \begin{align*}
        f(x_{t+1}') \leq f(x_{t+1}) \leq f(x_{t})
    \end{align*}
    \item {\bf Convergence guarantee.} For any $\epsilon > 0$, there exists a positive integer $T$ that
    \begin{align*}
        \min_{t \in [T]}\{ \| \nabla f(x_t) \} \leq \epsilon
    \end{align*}
\end{itemize}
At this point, we will provide a detailed introduction to our methods and work in the following sections.

\section{Regrettable Delta-Bar-Delta Algorithm}

This section introduces a novel learning rate schedule algorithm named the Regrettable Delta-Bar-Delta (RDBD) algorithm. It is developed based on the delta-bar-delta algorithm but with the addition of a unique feature that allows the learning rate schedule to undo the previous learning rate update when it is found to be biased. Building upon the delta-bar-delta algorithm, we analyze from the perspective of single-step steeper loss descent: For a neural network, we divide it into several vector parameters, each with its own learning rate. We employ the delta-bar-delta algorithm independently for optimizing each vector parameter. 

Let's consider a neural network with vector parameter weight $x_0$ at the initial step ($0$-th step). We also initialize a learning rate $\alpha_0$ for optimizing $x$ and a learning rate $\eta$ for optimizing $\alpha$. The loss function of $x$ is denoted as $f(x)$, and we aim to optimize $x$ over $T$ steps. Following the delta-bar-delta algorithm, at step $t$ for $t \in [T]$, we define the weight update of this step as $g_t$. We update the learning rate as $\alpha_t := \alpha_{t-1} + \eta \langle g_t, g_{t-1} \rangle$, where $g_{t-1}$ represents the weight update in the previous step. To proceed, we assume that the gradient of $f(x)$ is Lipschitz-continuous with a constant $L$. This assumption implies that for any $x, y \in \mathbb{R}^d$, the gradient difference can be bounded as $\| \nabla f(x) - \nabla f(y) \|_2 \leq L\| x - y \|_2$, which leads an obvious inequality
\begin{align*}
    f(x) \leq f(y) + \langle f(y), x - y \rangle + \frac{L}{2} \| x - y \|_2^2
\end{align*}
Hence, we can now analyze step $t$ ($t \in [T]$) of the optimization process. Suppose we fix $x_t$ and denote $x_{t+1} := x_t - \alpha_{t-1} g_t$. Additionally, let $x_{t+1}' = x_t - \alpha_t g_t$. We can compute the following inequalities:
\begin{align*}
    f(x_{t+1}')
    \leq & ~ f(x_{t+1}) + \langle \nabla f(x_{t+1}), x_{t+1}' - x_{t+1} \rangle + D \\
    \leq & ~ f(x_{t+1}) + \langle \nabla f(x_{t+1}), (\alpha_t - \alpha_{t-1}) g_t \rangle + D \\
    \leq & ~ f(x_{t+1}) - \langle \nabla f(x_{t+1}), \eta \langle g_{t}, g_{t-1} \rangle g_t \rangle + D \\
     \leq & ~ f(x_{t+1}) - \eta \langle g_{t}, g_{t-1} \rangle \langle \nabla f(x_{t+1}), g_t \rangle + D
\end{align*}
where $D := \frac{L}{2} \| x_{t+1}' - x_{t+1} \|_2^2$. From the above inequalities, we can observe that to ensure $f(x_{t+1}') \leq f(x_{t+1})$, we need the term $- \eta \langle g_{t}, g_{t-1} \rangle \langle \nabla f(x_{t+1}), g_t \rangle + \frac{L}{2} \| x_{t+1}' - x_{t+1} \|_2^2$ to be negative. Since $\frac{L}{2} \| x_{t+1}' - x_{t+1} \|_2^2 \geq 0$, the condition simplifies to $\eta \langle g_{t}, g_{t-1} \rangle \langle \nabla f(x_{t+1}), g_t \rangle \geq 0$. Moreover, the following equality holds:
\begin{align*}
    \langle g_{t}, g_{t-1} \rangle \langle \nabla f(x_{t+1}), g_t \rangle
    = & ~ \langle g_{t}, g_{t-1} \rangle \langle \E[g_{t+1}], g_t \rangle \\
    = & ~ \E[\langle g_{t}, g_{t-1} \rangle \langle g_{t+1}, g_t \rangle]
\end{align*}
In light of the proven loss descent bound, we only apply the delta-bar-delta algorithm when $\langle g_{t}, g_{t-1} \rangle \langle \nabla g_{t+1}, g_t \rangle \geq 0$. However, since we only get $g_{t+1}$ in the step $t+1$, but the update of $\alpha_t$ will be already updated at that time. In other words, we could not verify the unbiasedness of $t$-th learning rate update until we have $g_{t+1}$. So we give the algorithm a reversible ability: At step $t+1$, before updating the learning rates and weights, we incorporate a verification process. Specifically, we examine whether the condition $\langle g_{t}, g_{t-1} \rangle \langle \nabla g_{t+1}, g_t \rangle < 0$ holds. If this condition is satisfied, we proceed with the following update steps:
\begin{align*}
    & ~ x_{t} \leftarrow x_{t} + \eta \langle g_{t}, g_{t-1} \rangle g_{t-1} \\
    & ~ \alpha_{t} \leftarrow \alpha_{t} - \eta \langle g_{t}, g_{t-1} \rangle
\end{align*}
The above operation is used to correct the bias update of the learning rate in the previous step, which is referred to as the "Regrettable" algorithm because it can reverse the update when it is determined to be biased. The term "Regrettable" accurately reflects the nature of this algorithm and its ability to address biased updates.

In Section~\ref{sub:defs}, we formally provide the definitions and the algorithm of our RDBD algorithm. In Section~\ref{sub:steeper_descent}, we show our result that confirms the RDBD algorithm can improve convergence speed. In Section~\ref{sub:convergence}, we confirm our result that proves the convergence of the RDBD algorithm.

\subsection{Definitions and Algorithm}\label{sub:defs}

Here, we first define the dynamical learning rate $\alpha_t$ at step $t$.
\begin{definition}
    At $t$ step, given the parameter update $g_t \in \R^d$ of this step and parameter update $g_{t-1} \in \R^d$ of the previous step, we run RDBD on with loss function $f(x)$ with parameter $x \in \R^d$ and learning rates $\alpha$ and $\eta$, we optimize
    \begin{itemize}
        \item {\bf Part 1.} denote $g_{t+1} \in \R^d$ as the weight update in step $t+1$, if $\langle g_{t+1}, g_t \rangle \langle g_{t}, g_{t-1} \rangle \geq 0$, we have
        \begin{align*}
            \alpha_t = \alpha_{t-1} + \eta \langle g_{t}, g_{t-1} \rangle
        \end{align*}
        \item {\bf Part 2.} denote $g_{t+1} \in \R^d$ as the weight update in step $t+1$, if $\langle g_{t+1}, g_t \rangle \langle g_{t}, g_{t-1} \rangle < 0$, we have
        \begin{align*}
            \alpha_t = \alpha_{t-1}
        \end{align*}
    \end{itemize}
    for all $g_t$, $g_t$, $g_{t+1}$ is an unbiased estimator of gradient of $f(x_t)$ as $\E[g_t] = \nabla f(x_t)$.
\end{definition}

Next, we define our Regrettable Delta-Bar-Delta algorithm as follows:
\begin{definition}
    At $t$ step, given the parameter update $g_t \in \R^d$ of this step and parameter update $g_{t-1} \in \R^d$ of the previous step, we run RDBD on with loss function $f(x)$ with parameter $x \in \R^d$ and learning rates $\alpha$ and $\eta$, we optimize
    \begin{itemize}
        \item {\bf Part 1.} denote $g_{t+1} \in \R^d$ as the weight update in step $t+1$, if $g_{t-1}, g_t, g_{t+1}$ satisfy that $\langle g_{t+1}, g_t \rangle \langle g_{t}, g_{t-1} \rangle \geq 0$
        \begin{align*}
            x_{t+1}^{\rm RDBD} = & ~ x_{t} - ( \alpha + \eta g_t^\top g_{t-1}) \cdot g_t \\
            = & ~ x_{t} - \alpha g_t - \eta g_t^\top g_{t-1} g_t
        \end{align*}
        \item {\bf Part 2.} denote $g_{t+1} \in \R^d$ as the weight update in step $t+1$, if $g_{t-1}, g_t, g_{t+1}$ satisfy that $\langle g_{t+1}, g_t \rangle \langle g_{t}, g_{t-1} \rangle < 0$
        \begin{align*}
            x_{t+1}^{\rm RDBD} = & ~ x_{t} - \alpha g_t
        \end{align*}
    \end{itemize}
    for all $g_t$, $g_t$, $g_{t+1}$ is an unbiased estimator of gradient of $f(x_t)$ as $\E[g_t] = \nabla f(x_t)$.
\end{definition}

We provide the Regrettable Delta-Bar-Delta algorithm in Algorithm~\ref{alg:RDBD}.

\begin{algorithm}[!ht]\caption{Regrettable delta-bar-delta algorithm (RDBD)}\label{alg:RDBD}
\begin{algorithmic}[1]
\Statex \textbf{Input: } Function $f \in \mathcal{F}, f(x): \R^d \rightarrow \R^d$, parameters $x_0 \in \R^d$, initial learning rate $\alpha_0$, learning rate for step-size optimization $\eta$

\State $g_0 \leftarrow {\bf 0}_d$

\State $h_0 \leftarrow 0$

\State $t \leftarrow 0$

\While{$x_t$ not converged}

\State $t \leftarrow t + 1$

\State $g_t \leftarrow \nabla f_{\xi_t}(x_t)$

\State $h_t \leftarrow \langle g_t, g_{t-1} \rangle$

\If{$h_t \cdot h_{t-1} < 0$}

\State $x_{t-1} \leftarrow x_{t-1} + \eta h_{t-1} g_{t-1}$

\State $\alpha_{t-1} \leftarrow \alpha_{t-1} - \eta h_{t-1}$

\EndIf

\State $\alpha_t \leftarrow \alpha_{t-1} + \eta h_t$

\State $x_t \leftarrow x_{t-1} - \alpha_t g_t$

\EndWhile

\end{algorithmic}
\end{algorithm}

\subsection{Steeper Descent Guarantee}\label{sub:steeper_descent}

We provide our result lemma that confirms the faster loss descent when applying the RDBD algorithm as the learning rate schedule algorithm below.
\begin{lemma}[Steeper descent guarantee of RDBD, informal version of Lemma~\ref{lem:steeper_descent:formal}]\label{lem:steeper_descent:informal}
    Suppose given learning rates $\alpha$ and $\eta$, given loss function $f(x)$ with parameters $x$. Suppose $\nabla f(x)$ is Lipschitz-smooth with constant $L$ $\| \nabla f(x) - \nabla f(y) \|_2 \leq L \| x - y \|, \forall x, y \in \R^d$ and weight update in the $t$-th step $g_t$ is $\sigma$-bounded that $\| g_t \|_2 \leq \sigma, \forall t \in [T]$. Suppose $g_t$ is unbiased gradient estimator that $\E[g_t] = \nabla f(x_t)$ for $t \in [T]$ and $\langle g_t, \nabla f(x_t) \rangle \geq \mu \| g_t \|_2^2$.

    At $t$ step optimization, given weight update $g_t$ of $t$ step, previous step weight update $g_{t-1}$ of $t-1$ step, next step weight update $g_{t+1}$ of $t+1$ step. 
    
    Let $x_{t+1} = x_t - \alpha g_t$, let $x_{t+1}^{\rm RDBD} = x_{t} - \alpha g_t - \eta \langle g_t, g_{t-1} \rangle g_t$. 
    
    
    We can show that
    \begin{itemize}
        \item {\bf Part 1.} When $\langle g_{t+1}, g_t \rangle \langle g_{t}, g_{t-1} \rangle \geq 0$, and if $\eta \leq \frac{2 }{L \sigma^2 }$, we have
        \begin{align*}
            f(x_{t+1}^{\rm RDBD}) \leq f(x_{t+1})
        \end{align*}
        \item {\bf Part 2.} When $\langle g_{t+1}, g_t \rangle \langle g_{t}, g_{t-1} \rangle < 0$, we have
        \begin{align*}
            f(x_{t+1}^{\rm RDBD}) = f(x_{t+1})
        \end{align*}
        \item {\bf Part 3.} When $\alpha \leq \frac{2 \mu }{L}$, we have
        \begin{align*}
            f(x_{t+1}) \leq f(x_t)
        \end{align*}
    \end{itemize}
\end{lemma}

In Lemma~\ref{lem:steeper_descent:informal}, we demonstrate the steeper descent guarantee of our method that improves the convergence speed. Besides, we provide rigorous bounds for the variables $\eta$ and $\alpha_t$ during the training process, enabling more informed decisions regarding the initial value of $\alpha_0$ and $\eta$. For the proof of Lemma~\ref{lem:steeper_descent:formal}, please refer to Appendix~\ref{sub:steeper_descent_proof}.

\subsection{Minimizing Guarantee}\label{sub:convergence}

We present our results that confirm the convergence of the Regrettable Delta-Bar-Delta (RDBD) algorithm in mini-batch optimization. The convergence guarantee is stated informally as follows:

\begin{theorem}[Convergence guarantee for Regrettable Delta-Bar-Delta algorithm in mini-batch optimization, informal version of Theorem~\ref{thm:convergence_rdbd:formal}]\label{thm:convergence_rdbd:informal}
    Suppose $\nabla f(x)$ is Lipschitz-smooth with constant $L$ $\| \nabla f(x) - \nabla f(y) \|_2 \leq L \| x - y \|, \forall x, y \in \R^d$ and weight update in the $t$-th step $g_t$ is $\sigma$-bounded that $\| g_t \|_2 \leq \sigma, \forall t \in [T]$. Suppose $g_t$ is unbiased gradient estimator that $\E[g_t] = \nabla f(x_t)$ for $t \in [T]$. We initialize $\alpha_0 = \frac{\sqrt{f(x_0) - f^*}}{\sigma\sqrt{L T}}$, $\eta = \frac{\gamma\sqrt{f(x_0) - f^*}}{T\sigma^3\sqrt{L T}}$ where $\gamma \in (0, 1)$ is denoted as a scalar. We run the  Regrettable Delta-Bar-Delta algorithm at most
    \begin{align*}
        T = \frac{1}{\epsilon^2} \cdot \sigma \sqrt{L(f(x_0) - f^*)} (\frac{1}{1- \gamma} + \frac{1}{2} ( 1 + \gamma )) 
    \end{align*}
    times iterations, we have
    \begin{align*}
        \min_t \{ \| \nabla f(x_t) \|_2 \} \leq \epsilon
    \end{align*}
\end{theorem}
This theorem establishes the convergence of the RDBD algorithm in mini-batch optimization, providing a clear and professional statement of the result. Please see Appendix~\ref{sub:convergence_proof} for the proof of Theorem~\ref{thm:main_result:informal}.

\begin{figure*}
    \centering
    \includegraphics[width=\textwidth]{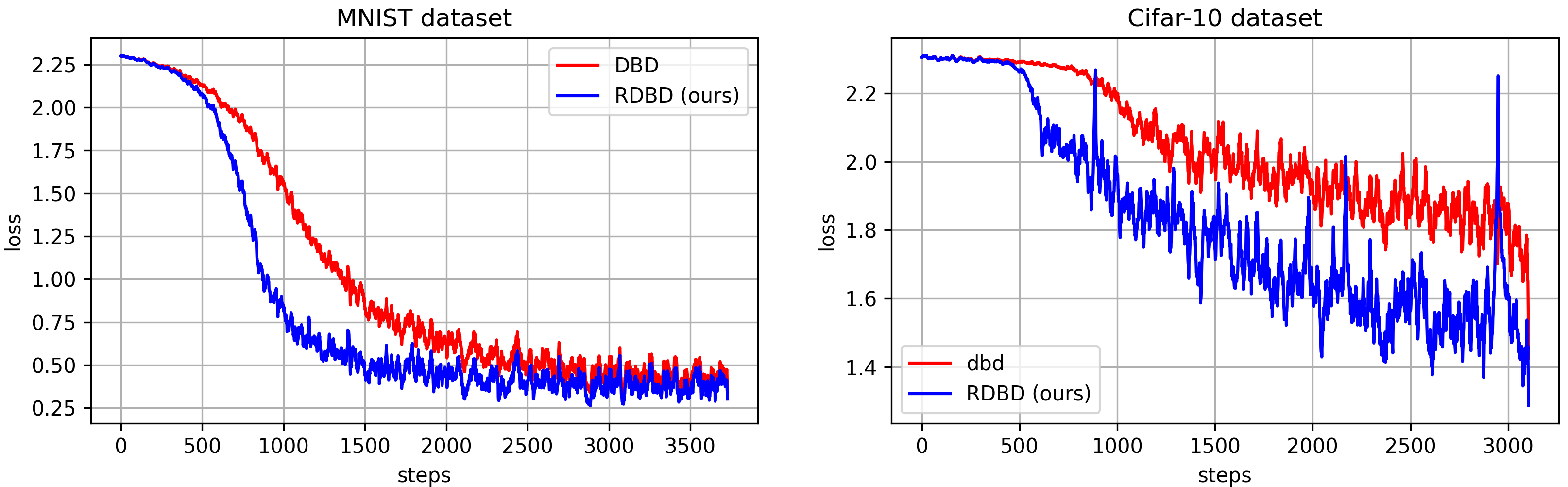}
    \caption{Comparison of the DBD algorithm and the RDBD algorithm on the Cifar-10 dataset and the MNIST dataset. }
    \label{fig:reduce}
\end{figure*}

\section{Experimental Results}\label{sec:experiment}

In this section, we provide a thorough overview of our experimental findings, which are categorized into two distinct sections to enhance clarity and organization. Section~\ref{sub:faster_speed} primarily emphasizes the accelerated convergence achieved through our method. We demonstrate the improved speed of convergence in comparison to alternative approaches. Section~\ref{sub:comparison_dbd_rdbd} delves into the ability of the RDBD algorithm to mitigate biased updates to the learning rate. We accomplish this by comparing the training loss of our algorithm with that of the delta-bar-delta algorithm. This comprehensive analysis offers valuable insights into the performance and effectiveness of our approach, contributing to a more comprehensive understanding of its practical implications. By presenting these results, we aim to provide a professional and coherent account of our experimental findings, ensuring clarity and facilitating a deeper understanding of our methodology.

For the details and more results of our experiment, please see Appendix~\ref{app:experiment}.

\subsection{Accelerating Convergence Speed of Adam and SGD with RDBD Algorithm}\label{sub:faster_speed}

In this experiment, we investigate the effectiveness of applying the RDBD learning rate schedule algorithm to both Adam and SGD optimization methods, and evaluate their performance on two widely used datasets: MNIST and Cifar-10.

To analyze the impact of the RDBD algorithm on the convergence speed of these optimization methods, we present the results in Figure~\ref{fig:main_result}. Specifically, we focus on the initial stages of the training process, examining the loss reduction achieved within the first 3750 steps for the MNIST dataset and the first 3125 steps for the Cifar-10 dataset.

From the observations made in Figure~\ref{fig:main_result}, it becomes evident that the RDBD algorithm plays a pivotal role in significantly accelerating the convergence speed of both SGD and Adam optimization methods. The reduction in loss achieved within the initial steps of training demonstrates the effectiveness of the RDBD algorithm in facilitating quicker convergence and improving overall optimization performance. Our results highlight the potential of the RDBD algorithm as a valuable learning rate schedule algorithm for enhancing the convergence speed in the context of optimization.

\subsection{Reducing Biased Learning Rate Updates with RDBD Algorithm}\label{sub:comparison_dbd_rdbd}

This experimental study aims to evaluate the effectiveness of our RDBD algorithm in mitigating biased learning rate updates when utilizing the delta-bar-delta algorithm for learning rate scheduling. Additionally, we focus on assessing the impact of these updates on loss reduction within the initial stages of training on the MNIST and Cifar-10 datasets.

Figure~\ref{fig:reduce} visually presents the results, indicating that the RDBD algorithm significantly suppresses inaccurate learning rate updates. As a consequence, it enables a more pronounced descent during the training process, leading to improved convergence.





\section{Conclusion}

In this study, we propose the Regrettable Delta-Bar-Delta (RDBD) algorithm as a novel learning rate schedule method. Our RDBD algorithm builds upon the classical delta-bar-delta algorithm, aiming to address the convergence issue caused by noise gradient in mini-batch optimization. By promptly reverting incorrect learning rate updates, RDBD facilitates a steeper descent of loss and faster convergence.

To establish the effectiveness of RDBD, we provide rigorous proofs and outline the conditions under which it achieves a steeper descent and faster convergence. Through comprehensive experiments, we demonstrate that the RDBD algorithm serves as an excellent learning rate schedule technique. It seamlessly integrates with any optimization algorithm, allowing for plug-and-play usage and direct acceleration of convergence.

The results of our study highlight the potential of the RDBD algorithm to significantly enhance the training performance of deep learning models. By mitigating the impact of noise gradient and enabling more precise learning rate updates, RDBD offers a promising approach for improving the efficiency and effectiveness of optimization algorithms.

\ifdefined\isarxiv

\else
\bibliography{ref}
\bibliographystyle{plainnat}
\section*{Checklist}

 \begin{enumerate}

 \item For all models and algorithms presented, check if you include:
 \begin{enumerate}
   \item A clear description of the mathematical setting, assumptions, algorithm, and/or model. [Yes]
   \item An analysis of the properties and complexity (time, space, sample size) of any algorithm. [Yes]
   \item (Optional) Anonymized source code, with specification of all dependencies, including external libraries. [Yes]
 \end{enumerate}

 \item For any theoretical claim, check if you include:
 \begin{enumerate}
   \item Statements of the full set of assumptions of all theoretical results. [Yes]
   \item Complete proofs of all theoretical results. [Yes]
   \item Clear explanations of any assumptions. [Yes]     
 \end{enumerate}

 \item For all figures and tables that present empirical results, check if you include:
 \begin{enumerate}
   \item The code, data, and instructions needed to reproduce the main experimental results (either in the supplemental material or as a URL). [Not Applicable]
   \item All the training details (e.g., data splits, hyperparameters, how they were chosen). [Not Applicable]
         \item A clear definition of the specific measure or statistics and error bars (e.g., with respect to the random seed after running experiments multiple times). [Not Applicable]
         \item A description of the computing infrastructure used. (e.g., type of GPUs, internal cluster, or cloud provider). [Not Applicable]
 \end{enumerate}

 \item If you are using existing assets (e.g., code, data, models) or curating/releasing new assets, check if you include:
 \begin{enumerate}
   \item Citations of the creator If your work uses existing assets. [Not Applicable]
   \item The license information of the assets, if applicable. [Not Applicable]
   \item New assets either in the supplemental material or as a URL, if applicable. [Not Applicable]
   \item Information about consent from data providers/curators. [Not Applicable]
   \item Discussion of sensible content if applicable, e.g., personally identifiable information or offensive content. [Not Applicable]
 \end{enumerate}

 \item If you used crowdsourcing or conducted research with human subjects, check if you include:
 \begin{enumerate}
   \item The full text of instructions given to participants and screenshots. [Not Applicable]
   \item Descriptions of potential participant risks, with links to Institutional Review Board (IRB) approvals if applicable. [Not Applicable]
   \item The estimated hourly wage paid to participants and the total amount spent on participant compensation. [Not Applicable]
 \end{enumerate}

 \end{enumerate}
\fi

\newpage
\onecolumn
\appendix

\section*{Appendix}

{\bf Roadmap.} We present a number of notations and definitions in Appendix~\ref{app:preli}. In Appendix~\ref{sec:convergence_dbd}, presents a basic derivation and convergence proof of the Delta-Bar-Delta Algorithm. We present proofs for the Regrettable Delta-Bar-Delta Algorithm in Appendix~\ref{app:main_proof}. We present more experiments in Appendix~\ref{app:experiment}.

\section{Preliminary}\label{app:preli}

In Appendix~\ref{sub:notations}, we provide the notation we use in this paper. In Appendix~\ref{sub:def_ass}, we provide some formal definition and assumption for our proofs.

\subsection{Notations}\label{sub:notations}

In this paper, we adopt the following notations: The number of dimensions of a neural network is denoted by $d$. The set of real numbers is represented by $\R$. A vector with $d$ elements, where each entry is a real number, is denoted as $x \in \mathbb{R}^{d}$. The weight vector parameter in a neural network is denoted as $x$. We use $x$ to denote the weight of a vector parameter in a neural network. For any positive integer $n$, we use $[n]$ to denote $\{1,2,\cdots, n\}$. The $\ell_p$ norm of a vector $x$ is denoted as $\| x \|_p$, for examples, $\| x \|_1 := \sum^n_{i=1} | x_i |$, $\| x \|_2 := ( \sum^n_{i=1} x_i^2 )^{1/2}$ and $\| x \|_\infty := \max_{i \in [n]} | x_i |$. For two vectors $x, y \in \R^d$, we denote $\langle x, y \rangle = \sum^n_{i=1}$ for $i \in [d]$. The loss function of a weight vector parameter $x \in \mathbb{R}^d$ on the entire dataset is denoted as $f: \mathbb{R}^d \rightarrow \mathbb{R}$. Specifically, $f_\xi(x)$ represents the loss function of $x$ on the batch data $\xi$. The learning rate during training is denoted as $\alpha$. In particular, $\eta$ is used to represent the learning rate for updating $\alpha$ in the delta-bar-delta algorithm. We consider the optimization process as consisting of $T$ update steps. For an integer $t \in [T]$, at step $t$, we have the weight vector parameter $x_t$ and the learning rate $\alpha_t$. We use $\nabla f(x)$ to denote $\frac{\partial f(x)}{\partial x}$. We use $\E[]$ to denote expectation. We use $\Pr[]$ to denote the probability.

\subsection{Definition and Assumptions}\label{sub:def_ass}

\begin{definition}\label{def:f_xi}
    Given a block of loss function $f: \R^d \rightarrow \R^d$, suppose we sample batch data $\xi$ from training set, we denote loss of $f$ with batch data $\xi$ and parameters $x$ as $f_\xi(x)$.
\end{definition}

\begin{definition}\label{def:f}
    Given a block of loss function $f: \R^d \rightarrow \R^d$, suppose we sample batch data $\xi$ from training set, we denote loss of $f$ with all data in training set as follows:
    \begin{align*}
        f(x) := \frac{1}{n} \cdot \sum_{i=1}^n f_\xi(x)
    \end{align*}
    where $n$ is number of $\xi$ in training set.
\end{definition}

\begin{definition}\label{def:g}
    Given a block of loss function $f: \R^d \rightarrow \R^d$, let $f(x)$ be defined as Definition~\ref{def:f}, at $t$ step, we have parameter $x_t \in \R^d$, if a weight update $g_t$ satisfies that $\E[g_t] = \nabla f(x_t)$, then we say $g_t$ is an unbiased estimator of gradient of $f(x_t)$.
\end{definition}

\begin{definition}\label{def:rdbd}
    At $t$ step, given the parameter update $g_t \in \R^d$ of this step and parameter update $g_{t-1} \in \R^d$ of the previous step, we run RDBD on with loss function $f(x)$ with parameter $x \in \R^d$ and learning rates $\alpha$ and $\eta$, we optimize
    \begin{itemize}
        \item {\bf Part 1.} given $g_{t+1} \in \R^d$, if $g_{t-1}, g_t, g_{t+1}$ satisfy that $\langle g_{t+1}, g_t \rangle \langle g_{t}, g_{t-1} \rangle \geq 0$
        \begin{align*}
            x_{t+1}^{\rm RDBD} = & ~ x_{t} - ( \alpha + \eta g_t^\top g_{t-1}) \cdot g_t \\
            = & ~ x_{t} - \alpha g_t - \eta g_t^\top g_{t-1} g_t
        \end{align*}
        \item {\bf Part 2.} given $g_{t+1} \in \R^d$, if $g_{t-1}, g_t, g_{t+1}$ satisfy that $\langle g_{t+1}, g_t \rangle \langle g_{t}, g_{t-1} \rangle < 0$
        \begin{align*}
            x_{t+1}^{\rm RDBD} = & ~ x_{t} - \alpha g_t
        \end{align*}
    \end{itemize}
    for all $g_t$, $g_t$ is an unbiased estimator of gradient of $f(x_t)$ as Definition~\ref{def:g}.
\end{definition}

\begin{assumption}\label{ass:lipschitz}
    We assume that loss function is Lipschitz-smooth with constant $L$, which is
    \begin{align*}
        \| \nabla f(x) - \nabla f(y) \|_2 \leq L \cdot \| x - y \|_2, \forall x, y \in \R^d
    \end{align*}
\end{assumption}

\begin{assumption}\label{ass:upper_bound_grad_f}
    We assume that given a weight update $g_t$ as Definition~\ref{def:g}, $g_t$ is $\sigma$-bounded that
    \begin{align*}
        \| g_t \|_2 \leq \sigma
    \end{align*}
\end{assumption}

\begin{corollary}\label{cor:lipschitz}
    Since $f(x)$ satisfies Assumption~\ref{ass:lipschitz}, we have
    \begin{align*}
        f(x) - ( f(y) + \langle \nabla f(y) (x - y) \rangle \leq \frac{1}{2} L \| x - y \|_2^2
    \end{align*}
\end{corollary}

\begin{assumption}\label{ass:lower_bound_langle_nabla_f_g_rangle}
    Let $g_t$ be defined as Definition~\ref{def:g}, we assume that $g_t$ satisfies 
    \begin{align*}
        \langle \nabla f(x_t), g_t \rangle \geq \mu \| g_t \|_2^2
    \end{align*}
    where $\mu > 0$ is a scalar.
\end{assumption}

\section{Delta-Bar-Delta Algorithm}\label{sec:convergence_dbd}

Appendix~\ref{sub:dbd_def} provides a formal definition of the delta-bar-delta algorithm. And we provide a basic to derive the delta-bar-delta algorithm in Appendix~\ref{sub:basic_derivation}. We show our result and proof for convergence of the delta-bar-delta algorithm in Appendix~\ref{sub:dbd_convergence_proof}.

\subsection{Definition}\label{sub:dbd_def}

\begin{definition}
    Given a loss function $f: \R^d \rightarrow \R$, given a weight of parameter $x_0$. Given learning rates $\alpha_0$ and $\eta$, we denote $\nabla f(x_0) = \mathbf{0}_d$, then at $t$ step optimization, we run the delta-bar-delta algorithm as follows:
    \begin{align*}
        & ~ \alpha_t = \alpha_{t-1} + \eta \langle \nabla f(x_t), \nabla f(x_0) \rangle
        & ~ x_t = x_{t-1} - \alpha_t \nabla f(x_t) 
    \end{align*}
\end{definition}

\subsection{Basic Derivation}\label{sub:basic_derivation}

\begin{lemma}
    Let $f(x)$ be defined as Definition~\ref{def:f}, at $t-1$ step, we have $x_t = x_{t-1} - \alpha_{t-1} \nabla f(x_{t-1})$, then at $t$ step, we have
    \begin{align*}
        \nabla_{\alpha_{t-1}} f(x_t) = - \langle \nabla f(x_t), \nabla f(x_{t-1}) \rangle
    \end{align*}
\end{lemma}

\begin{proof}
    We have
    \begin{align*}
        \nabla_{\alpha_{t-1}} f(x_t)
        = & ~ \langle \nabla f(x_t), \nabla_{\alpha_{t-1}} x_t \rangle \\
        = & ~ \langle \nabla f(x_t), \nabla_{\alpha_{t-1}} (x_{t-1} - \alpha_{t-1} \nabla f(x_{t-1})) \rangle \\
        = & ~ \langle \nabla f(x_t), - \nabla f(x_{t-1}) \rangle \\
        = & ~ - \langle \nabla f(x_t), \nabla f(x_{t-1}) \rangle \\
    \end{align*}
    where the first equality follows from simple differential rule, the second equality follows from $x_t = x_{t-1} - \alpha_{t-1} \nabla f(x_{t-1})$, the third equality follows from simple differential rule, the last equality follows from simple algebra.
\end{proof}

\subsection{Convergence Guarantee}\label{sub:dbd_convergence_proof}

\begin{theorem}[Convergence guarantee for delta-bar-delta algorithm in mini-batch optimization, formal version of Theorem~\ref{thm:convergence_dbd:informal}]\label{thm:convergence_dbd:formal}
    Suppose $\nabla f(x)$ is Lipschitz-smooth with constant $L$ as Assumption~\ref{ass:lipschitz} and gradient of loss function $\nabla f(x)$ is $\sigma$-bounded that $\| \nabla f(x ) \|_2 \leq \sigma^2$. we initialize $\alpha_0 = \frac{1}{L}$, $\eta = \frac{\gamma}{T\sigma^2L}$, where $\gamma \in [0, 1)$ is denoted as a scalar. We run the delta-bar-delta algorithm at most $\frac{2 L( f(x_0) - f^* )}{(1 - \gamma^2) \epsilon^2}$ times iterations, we have
    \begin{align*}
        \min_t \{ \| \nabla f(x_t) \|_2 \} \leq \epsilon
    \end{align*}
\end{theorem}

\begin{proof}
    We have
    \begin{align}\label{eq:ineq_dbd}
        f(x_{t+1}) 
        \leq & ~ f(x_t) + \langle \nabla f(x_t), x_{t+1} - x_t \rangle + \frac{1}{2} L \| x_{t+1} - x_t \|_2^2 \notag \\
        = & ~ f(x_t) - \alpha_t \| \nabla f(x_t) \|_2^2 + \frac{1}{2} L \| \alpha_t \nabla f(x_t) \|_2^2 \notag \\
        = & ~ f(x_t) +  \frac{\alpha_t^2 L - 2\alpha_t}{2}  \| \nabla f(x_t) \|_2^2 
    \end{align}
    where the first equality follows from Corollary~\ref{cor:lipschitz}, the second equality follows from $x_{t+1} = x_t - \alpha_t \nabla f(x_t)$, the third equality follows from simple algebra.

    We obtain
    \begin{align}\label{eq:upper_bound_grad_f^2_norm}
        \sum_{t=1}^T \| \nabla f(x_t) \|_2^2 
        \leq & ~ \sum_{t=1}^T \frac{2}{2 \alpha_t - \alpha_t^2 L} ( f(x_t) - f(x_{t+1}) ) \notag \\
        \leq & ~ \frac{2L}{1 - \gamma^2} \sum_{t=1}^T ( f(x_t) - f(x_{t-1}) ) \notag \\
        = & ~ \frac{2L}{1 - \gamma^2} ( f(x_0) - f^* )
    \end{align}
    where the first inequality follows from Eq.~\eqref{eq:ineq_dbd}, the second inequality follows from Lemma~\ref{lem:upper_bound_frac_2_2alpha_sub_alpha^2_L}, the last equality follows from $\sum_{t=1}^T ( f(x_t) - f(x_{t-1}) ) = f(x_0) - f^*$.

    Hence, we can get
    \begin{align*}
        \min_t \{ \| \nabla f(x_t) \|_2 \} 
        = & ~ (\min_t \{ \| \nabla f(x_t) \|_2^2 \} )^{1/2}\\
        \leq & ~ ( \frac{1}{T} \sum_{t=1}^T \| \nabla f(x_t) \|_2^2 )^{1/2}\\
        \leq & ~ ( \frac{1}{T} \cdot \frac{2L}{1 - \gamma^2} ( f(x_0) - f^* )  )^{1/2}
    \end{align*}
    where the first and the second inequalities follow from simple algebras, the last inequality follows from Eq.\eqref{eq:upper_bound_grad_f^2_norm}.
\end{proof}

\subsection{Upper Bound on \texorpdfstring{$\frac{2}{2 \alpha_t - \alpha_t^2 L}$}{}}

\begin{lemma}\label{lem:upper_bound_frac_2_2alpha_sub_alpha^2_L}
    Let $\alpha_0 = \frac{1}{L}$ and $\eta = \frac{\gamma}{T\sigma^2L}$, where $\gamma \in [0, 1)$ is denoted as a scalar, then we have
    \begin{align*}
        \frac{2}{2 \alpha_t - \alpha_t^2 L} \leq \frac{2L}{1 - \gamma^2}
    \end{align*}
\end{lemma}

\begin{proof}
    We have
    \begin{align*}
        \frac{2}{2 \alpha_t - \alpha_t^2 L}
        \leq & ~ \frac{2}{2 (\frac{1}{L} - T \eta \sigma^2) - (\frac{1}{L} - T \eta \sigma^2)^2 L} \\
        = & ~ \frac{2}{\frac{2}{L} - 2T \eta \sigma^2 - \frac{1}{L} + 2 T \eta \sigma^2 - T^2 \eta^2 \sigma^4 L} \\
        = & ~ \frac{2L}{1 - T^2 \eta^2 \sigma^4 L^2}  \\
        = & ~ \frac{2L}{1 - \gamma^2}
    \end{align*}
    where the first inequality follows from Lemma~\ref{lem:bound_alpha} and simple algebra, the second and the third equalities follow from simple algebras, the last step follows from $\eta = \frac{\gamma}{T\sigma^2L}$.
\end{proof}

\subsection{Lower Bound and Upper Bound on \texorpdfstring{$\alpha_t$}{}}

\begin{lemma}\label{lem:bound_alpha}
    We initialize learning rate as $\alpha_0$, then we run DBD algorithm as $\alpha_t = \alpha_{t-1} + \eta \langle \nabla f(x_t), \nabla f(x_{t-1}) \rangle = \alpha_0 + \sum_{t=1}^T \eta \langle \nabla f(x_t), \nabla f(x_{t-1}) \rangle$, where $\eta$ is the learning rate to optimize $\alpha$ and we let $\nabla f(x_0) = \mathbf{0}_d$. We have
    \begin{align*}
        \alpha_0 - T \eta \sigma^2 \leq \alpha_t \leq \alpha_0 + T \eta \sigma^2
    \end{align*}
\end{lemma}

\begin{proof}
    We have
    \begin{align*}
        \alpha_t
        = & ~ \alpha_0 + \sum_{t=0}^T \eta \langle \nabla f(x_t), \nabla f(x_{t-1}) \rangle \\
        \geq & ~ \alpha_0 - T\eta \sigma^2
    \end{align*}
    where the first equality follows from the definition of $\alpha_t$, the second equality follows from simple algebra and $\| \nabla f(x ) \|_2 \leq \sigma^2$.

    Also, we have
    \begin{align*}
        \alpha_t
        = & ~ \alpha_0 + \sum_{t=0}^T \eta \langle \nabla f(x_t), \nabla f(x_{t-1}) \rangle \\
        \leq & ~ \alpha_0 + T \eta \sigma^2
    \end{align*}
    where the first equality follows from the definition of $\alpha_t$, the second equality follows from simple algebra and Assumption~\ref{ass:upper_bound_grad_f}.
\end{proof}

\section{Regrettable Delta-Bar-Delta (RDBD) Algorithm}\label{app:main_proof}

In Appendix~\ref{sub:main_result_proof}, we provide our main result of this paper. In Appendix~\ref{sub:steeper_descent_proof}, we provide our result that confirms our RDBD algorithm has steeper loss descent. In Appendix~\ref{sub:convergence_proof}, we show our result that confirms the convergence of our RDBD algorithm.

\subsection{Main Result}\label{sub:main_result_proof}

\begin{theorem}[Main result, formal version of Theorem~\ref{thm:main_result:informal}]\label{thm:main_result:formal}
    For a loss function $f: \R^d \rightarrow \R$ and the weight $x$ of a vector parameter of a model. We run RDBD algorithm on it. We first initialize $x_0 = x$, $\alpha_0 = \frac{\sqrt{f(x_0) - f^*}}{\sigma\sqrt{L T}}$, $\eta = \frac{\gamma\sqrt{f(x_0) - f^*}}{T\sigma^3\sqrt{L T}}$ where $\gamma \in (0, 1)$. Denote $f^* := \min_{x \in \R^d} f(x)$. Let $\epsilon > 0$ be denoted as the error of training. We have
    \begin{itemize}
        \item {\bf Steeper Loss Descent.} At step $t$, the RDBD algorithm accelerates loss reduction as follows:
        \begin{align*}
            f(x_{t+1}^{\rm RDBD}) \leq f(x_{t+1}) \leq f(x_t)
        \end{align*}
        \item {\bf Convergence. } Let
        \begin{align*}
            T = \frac{1}{\epsilon^2} \cdot \sigma \sqrt{L(f(x_0) - f^*)} (\frac{1}{1- \gamma} + \frac{1}{2} ( 1 + \gamma )) 
        \end{align*}
        then, at most $T$ time iterations RDBD algorithm, we have
        \begin{align*}
            \min_t \{ \| \nabla f(x_t) \|_2 \} \leq \epsilon
        \end{align*}
    \end{itemize}
\end{theorem}
\begin{proof}
    This proof is following from Lemma~\ref{lem:steeper_descent:formal} and Theorem~\ref{thm:convergence_rdbd:formal}.
\end{proof}

\subsection{Steeper Descent Guarantee}\label{sub:steeper_descent_proof}

\begin{lemma}[Steeper descent guarantee of RDBD, formal version of Lemma~\ref{lem:steeper_descent:informal}]\label{lem:steeper_descent:formal}
    Suppose given learning rates $\alpha$ and $\eta$, given loss function $f(x)$ with parameters $x$ as Definition~\ref{def:f}. 
    
    At $t$ step optimization, given weight update $g_t$ of $t$ step, previous step weight update $g_{t-1}$ of $t-1$ step, next step weight update $g_{t+1}$ of $t+1$ step. 
    
    Let $x_{t+1} = x_t - \alpha g_t$, let $x_{t+1}^{\rm RDBD}$ be defined as Definition~\ref{def:rdbd}. 
    
    Denote $\tau > 0$ that $\min_t\{\langle g_{t+1}, g_t \rangle\}\geq \tau$. Let $L$, $\sigma$, $\mu$ be defined as Assumption~\ref{ass:lipschitz}, Assumption~\ref{ass:upper_bound_grad_f}, Assumption~\ref{ass:lower_bound_langle_nabla_f_g_rangle}.
    
    We can show that
    \begin{itemize}
        \item {\bf Part 1.} When $\langle g_{t+1}, g_t \rangle \langle g_{t}, g_{t-1} \rangle \geq 0$, and if $\eta \leq \frac{2 }{L \sigma^2 }$, we have
        \begin{align*}
            f(x_{t+1}^{\rm RDBD}) \leq f(x_{t+1})
        \end{align*}
        \item {\bf Part 2.} When $\langle g_{t+1}, g_t \rangle \langle g_{t}, g_{t-1} \rangle < 0$, we have
        \begin{align*}
            f(x_{t+1}^{\rm RDBD}) = f(x_{t+1})
        \end{align*}
        \item {\bf Part 3.} When $\alpha \leq \frac{2 \mu }{L}$, we have
        \begin{align*}
            f(x_{t+1}) \leq f(x_t)
        \end{align*}
    \end{itemize}
\end{lemma}

\begin{proof}
    {\bf Proof of Part 1.} 
    For $t \in [T]$, we have
    \begin{align*}
        | \langle g_{t+1}, g_t \rangle |
        \geq & ~ \min_t\{\langle g_{t+1}, g_t \rangle\} \\
        \geq & ~ \tau \\
        \geq & ~ \frac{\tau}{\sigma^2} \| g_t \|_2^2
    \end{align*}
    where the first inequality follows from simple algebra, the second inequality follows from $\min_t\{\langle g_{t+1}, g_t \rangle\}\geq \tau$, the thrid inequality follows from Assumption~\ref{ass:upper_bound_grad_f}.

    Next, we have
    \begin{align}\label{eq:lower_bound_langle_grad_f_g_t_rangle}
        | \langle \nabla f(x_{t+1}), g_t \rangle |
        = & ~ | \langle \E[g_{t+1}], g_t \rangle | \notag \\
        \geq & ~ \frac{\tau}{\sigma^2} \| g_t \|_2^2
    \end{align}
    where the first equality follows from Definition~\ref{def:g}, the second inequality follows from simple algebra.

    Thus, we have
    \begin{align}\label{eq:descent_ineq}
        f(x_{t+1}^{\rm RDBD}) \leq & ~ f(x_{t+1}) + \langle \nabla f(x_{t+1}), x_{t+1}^{\rm RDBD} - x_{t+1} \rangle + \frac{1}{2} L \| x_{t+1}^{\rm RDBD} - x_{t+1} \|_2^2 \notag \\
        = & ~ f(x_{t+1}) - \langle \nabla f(x_{t+1}), \eta g_{t}^\top g_{t-1} g_t \rangle + \frac{1}{2} L \| \eta g_{t}^\top g_{t-1} g_t \|_2^2  \notag \\
        = & ~ f(x_{t+1}) - \eta g_{t}^\top g_{t-1} \langle \nabla f(x_{t+1}), g_t \rangle + \frac{1}{2} L \| \eta g_{t}^\top g_{t-1} g_t \|_2^2  \notag \\
        = & ~ f(x_{t+1}) - \eta g_{t}^\top g_{t-1} \langle \nabla f(x_{t+1}), g_t \rangle + \frac{1}{2} L ( \eta g_{t}^\top g_{t-1} )^2 \| g_t \|_2^2  \notag \\
        \leq & ~ f(x_{t+1}) - \eta | g_{t}^\top g_{t-1} | \cdot \frac{\tau}{\sigma^2} \| g_t \|_2^2 + \frac{1}{2} L ( \eta g_{t}^\top g_{t-1} )^2 \| g_t \|_2^2  \notag \\
        = & ~ f(x_{t+1}) + ( - \frac{\tau}{\sigma^2} + \frac{1}{2} \eta L | g_{t}^\top g_{t-1} | ) \cdot | \eta g_{t}^\top g_{t-1} | \| g_t \|_2^2
    \end{align}
    where the first inequality follows from Corollary~\ref{cor:lipschitz}, the second equality follows from Part 1 of Definition~\ref{def:rdbd}, the third and fourth equalities follow from simple algebra, the fifth inequality follows from Eq.~\eqref{eq:lower_bound_langle_grad_f_g_t_rangle}, the last equality follows from simple algebra.

    When $\eta \leq \frac{2 }{L \sigma^2 }$, we can show that 
    \begin{align}\label{eq:negetive_term}
        ( - \frac{\tau}{\sigma^2} + \frac{1}{2} \eta L | g_{t}^\top g_{t-1} | ) \cdot | \eta g_{t}^\top g_{t-1} | \| g_t \|_2^2 
        \leq & ~ ( - \frac{\tau}{\sigma^2} + \frac{1}{2} \frac{2}{L \sigma^2 } L | g_{t}^\top g_{t-1} | ) \cdot | \eta g_{t}^\top g_{t-1} | \| g_t \|_2^2 \notag \\
        \leq & ~ ( - \frac{\tau}{\sigma^2} + \frac{1}{2} \frac{2 \tau}{L \sigma^2 | g_{t}^\top g_{t-1} | } L | g_{t}^\top g_{t-1} | ) \cdot | \eta g_{t}^\top g_{t-1} | \| g_t \|_2^2 \notag \\
        \leq & ~ 0
    \end{align}
    where the first inequality follows from $\eta \leq \frac{2 }{L \sigma^2 }$, the second inequality follows from $\tau < | g_t^\top g_{t-1} |$, the last inequality follows from simple algebra and $| \eta g_{t}^\top g_{t-1} | \| g_t \|_2^2 \geq 0$.

    So we have
    \begin{align*}
        f(x_{t+1}^{\rm RDBD}) \leq & ~ f(x_{t+1}) + ( - \frac{\tau}{\sigma^2} + \frac{1}{2} \eta L | g_{t}^\top g_{t-1} | ) \cdot | \eta g_{t}^\top g_{t-1} | \| g_t \|_2^2 \\
        \leq & ~ f(x_{t+1})
    \end{align*}
    where the first inequality follows from Eq.~\eqref{eq:descent_ineq}, the second inequality follows from Eq.~\eqref{eq:negetive_term} and simple algebra.

    {\bf Proof of Part 2.} When $\langle g_{t+1}, g_t \rangle \langle g_{t}, g_{t-1} \rangle < 0$, we have
    \begin{align*}
        f(x_{t+1}^{\rm RDBD}) = f(x_{t+1})
    \end{align*}
    where this equality follows from $x_{t+1}^{\rm RDBD} = x_t - \alpha g_t = x_{t+1}$ (Part 2 of Definition~\ref{def:rdbd}).

    {\bf Proof of Part 3.} If $\alpha \leq \frac{2\mu}{L}$, then we have
    \begin{align*}
        f(x_{t+1}) 
        \leq & ~ f(x_t) + \langle \nabla f(x_t), x_{t+1} - x_t \rangle + \frac{1}{2} L \| x_{t+1} - x_t \|_2^2 \\
        = & ~ f(x_t) - \langle \nabla f(x_t), \alpha g_t \rangle + \frac{1}{2} L \| \alpha g_t \|_2^2 \\
        = & ~ f(x_t) - \alpha \langle \nabla f(x_t), g_t \rangle + \frac{1}{2} \alpha^2 L \| g_t \|_2^2 \\
        \leq & ~ f(x_t) - \alpha \mu \| g_t \|_2^2 + \frac{1}{2} \alpha^2 L \| g_t \|_2^2 \\
        = & ~ f(x_t) + ( - \alpha \mu + \frac{1}{2} \alpha^2 L) \| g_t \|_2^2 \\
        \leq & ~ f(x_t)
    \end{align*}
    where this inequality follows from Corollary~\ref{cor:lipschitz}, the second and the third equalities follow from simple algebra, the fourth inequality follows from Assumption~\ref{ass:lower_bound_langle_nabla_f_g_rangle}, the fifth equality follows from simple algebra, the last inequality follows from $\alpha \leq \frac{2 \mu }{L}$.

    Then, we combine those three parts, we can show that
    \begin{align*}
        f(x_{t+1}^{\rm RDBD}) \leq f(x_{t+1}) \leq f(x_t)
    \end{align*}
    thus, we complete the proof.
\end{proof}

\subsection{Convergence Guarantee}\label{sub:convergence_proof}

Here, we first define the dynamical learning rate at step $t$.
\begin{definition}\label{def:alpha_t}
    We run RDBD as Definition~\ref{def:rdbd}, at step $t$, we have
    \begin{itemize}
        \item {\bf Part 1.} if $\langle g_{t+1}, g_t \rangle \langle g_{t}, g_{t-1} \rangle \geq 0$, we have
        \begin{align*}
            \alpha_t = \alpha_{t-1} + \eta \langle g_{t}, g_{t-1} \rangle
        \end{align*}
        \item {\bf Part 2.} if $\langle g_{t+1}, g_t \rangle \langle g_{t}, g_{t-1} \rangle < 0$, we have
        \begin{align*}
            \alpha_t = \alpha_{t-1}
        \end{align*}
    \end{itemize}
\end{definition}

\begin{theorem}[Convergence guarantee for Regrettable Delta-Bar-Delta algorithm in mini-batch optimization, formal version of Theorem~\ref{thm:convergence_rdbd:informal}]\label{thm:convergence_rdbd:formal}
    Suppose $\nabla f(x)$ is Lipschitz-smooth with constant $L$ $\| \nabla f(x) - \nabla f(y) \|_2 \leq L \| x - y \|, \forall x, y \in \R^d$ and weight update in the $t$-th step $g_t$ is $\sigma$-bounded that $\| g_t \|_2 \leq \sigma, \forall t \in [T]$. Suppose $g_t$ is unbiased gradient estimator that $\E[g_t] = \nabla f(x_t)$ for $t \in [T]$. We initialize $\alpha_0 = \frac{\sqrt{f(x_0) - f^*}}{\sigma\sqrt{L T}}$, $\eta = \frac{\gamma\sqrt{f(x_0) - f^*}}{T\sigma^3\sqrt{L T}}$ where $\gamma \in (0, 1)$ is denoted as a scalar. We run the  Regrettable Delta-Bar-Delta algorithm at most
    \begin{align*}
        T = \frac{1}{\epsilon^2} \cdot \sigma \sqrt{L(f(x_0) - f^*)} (\frac{1}{1- \gamma} + \frac{1}{2} ( 1 + \gamma )) 
    \end{align*}
    times iterations, we have
    \begin{align*}
        \min_t \{ \| \nabla f(x_t) \|_2 \} \leq \epsilon
    \end{align*}
    
\end{theorem}

\begin{proof}
    We have
    \begin{align}\label{eq:ineq_f_x_t_add_1_RDBD}
        f(x_{t+1}^{\rm RDBD}) 
        \leq & ~ f(x_t) + \langle \nabla f(x_t), x_{t+1}^{\rm RDBD} - x_t \rangle + \frac{1}{2} L \| x_{t+1}^{\rm RDBD} - x_t \|_2^2 \notag \\
        = & ~ f(x_t) - \langle \nabla f(x_t), \alpha_t g_t \rangle + \frac{1}{2} L \| \alpha_t g_t \|_2^2 \notag \\
        = & ~ f(x_t) - \alpha_t \langle \nabla f(x_t), g_t \rangle + \frac{1}{2} L \alpha_t^2 \| g_t \|_2^2
    \end{align}
    where the first inequality follows from Corollary~\ref{cor:lipschitz}, the second equality follows from Definition~\ref{def:alpha_t}, the last equality follows from simple algebra.

    Hence, we can show that
    \begin{align}\label{eq:ineq_e_f_x_t_add_1_RDBD}
        \E[f(x_{t+1}^{\rm RDBD})]
        \leq & ~ \E[f(x_t) - \alpha_t \langle \nabla f(x_t), g_t \rangle + \frac{1}{2} L \alpha_t^2 \| g_t \|_2^2] \notag \\
        = & ~ f(x_t) - \alpha_t \langle \nabla f(x_t), \E[g_t] \rangle + \frac{1}{2} L \alpha_t^2 \E[\| g_t \|_2^2] \notag \\
        = & ~ f(x_t) - \alpha_t \| \nabla f(x_t) \|_2^2 + \frac{1}{2} L \alpha_t^2 \E[\| g_t \|_2^2] \notag \\
        \leq & ~ f(x_t) - \alpha_t \| \nabla f(x_t) \|_2^2 + \frac{1}{2} L \alpha_t^2 \sigma^2
    \end{align}
    where the first inequality follows from the expectation of Eq.\eqref{eq:ineq_f_x_t_add_1_RDBD}, the second equality follows from simple algebra, the third equality follows from Definition~\ref{def:g}, the last inequality follows from Assumption~\ref{ass:upper_bound_grad_f}.

    Let $\eta = \frac{\gamma\sqrt{f(x_0) - f^*}}{T\sigma^3\sqrt{L T}}$, $\alpha_0 = \frac{\sqrt{f(x_0) - f^*}}{\sigma\sqrt{L T}}$, we obtain
    \begin{align}\label{eq:upper_bound_e_nabla_f}
        \sum_{t=1}^T \| \nabla f(x_t) \|_2^2 \leq & ~ \sum_{t=1}^T (\frac{1}{\alpha_t}(f(x_t) - \E[f(x_{t+1}^{\rm RDBD})]) + \frac{1}{2} L \alpha_t^2 \sigma^2) \notag \\
        \leq & ~ \sum_{t=1}^T \frac{1}{\alpha_t}(f(x_t) - \E[f(x_{t+1}^{\rm RDBD})]) + \sum_{t=1}^T \frac{1}{2} L \alpha_t^2 \sigma^2 \notag\\
        = & ~ \frac{1}{\alpha_0 - T\eta\sigma^2}\sum_{t=1}^T (f(x_t) - \E[f(x_{t+1}^{\rm RDBD})]) + \sum_{t=1}^T \frac{1}{2} L \alpha_t^2 \sigma^2 \notag \\
        = & ~ \frac{1}{\alpha_0 - T\eta\sigma^2}( f(x_0) - f^*) + \sum_{t=1}^T \frac{1}{2} L \alpha_t^2 \sigma^2 \notag\\
        \leq & ~ \frac{1}{\alpha_0 - T\eta\sigma^2}( f(x_0) - f^*) + \frac{1}{2} L \sigma^2 (\alpha_0 + T\eta\sigma^2) \sum_{t=1}^T \alpha_t  \notag\\
        \leq & ~ \frac{1}{\alpha_0 - T\eta\sigma^2}( f(x_0) - f^*) + \frac{1}{2} L \sigma^2 \sum_{t=1}^T \alpha_t  \notag \\
        = & ~ \frac{1}{\alpha_0 - T\eta\sigma^2}( f(x_0) - f^*) + \frac{1}{2} L \sigma^2 \sum_{t=1}^T (\alpha_0 + t \eta \langle g_t, g_{t-1} \rangle ) \notag \\
        \leq & ~ \frac{1}{\alpha_0 - T\eta\sigma^2}( f(x_0) - f^*) + \frac{1}{2} L \sigma^2 \sum_{t=1}^T (\alpha_0 + t \eta \sigma^2 ) \notag \\
        \leq & ~ \frac{1}{\alpha_0 - T\eta\sigma^2}( f(x_0) - f^*) + \frac{1}{2} L \sigma^2 \int_1^T (\alpha_0 + t \eta \sigma^2 )\d t \notag \\
        \leq & ~ \frac{1}{\alpha_0 - T\eta\sigma^2}( f(x_0) - f^*) + \frac{1}{2} L \sigma^2 \cdot (\alpha_0(T-1) + \frac{1}{2}\eta \sigma^2 (T-1)^2) \notag \\
        \leq & ~ \frac{1}{\alpha_0 - T\eta\sigma^2}( f(x_0) - f^*) + \frac{1}{2} L \sigma^2 \cdot (\alpha_0 T + \frac{1}{2}\eta \sigma^2 T^2) \notag \\
        = & ~ \frac{1}{\alpha_0 - \frac{\gamma\sqrt{f(x_0) - f^*}}{\sigma \sqrt{L T}}}( f(x_0) - f^*) + \frac{1}{2} L \sigma^2 \cdot (\alpha_0 T + \frac{0.5 \gamma \sqrt{f(x_0) - f^*}}{\sigma \sqrt{L T}} T) \notag \\
        \leq & ~ \frac{1}{\alpha_0 - \frac{\gamma\sqrt{f(x_0) - f^*}}{\sigma \sqrt{L T}}}( f(x_0) - f^*) + \frac{1}{2} L \sigma^2 \cdot (\alpha_0 T + \frac{\gamma\sqrt{f(x_0) - f^*}}{\sigma \sqrt{L T}} T) \notag \\
        = & ~ \frac{1}{\frac{\sqrt{f(x_0) - f^*}}{\sigma\sqrt{LT}} - \frac{\gamma\sqrt{f(x_0) - f^*}}{\sigma \sqrt{L T}}}( f(x_0) - f^*) + \frac{1}{2} L \sigma^2 \cdot (\frac{\sqrt{f(x_0) - f^*}}{\sigma\sqrt{LT}} T + \frac{\gamma\sqrt{f(x_0) - f^*}}{\sigma \sqrt{L T}} T) \notag \\
        = & ~ \sigma \sqrt{LT(f(x_0) - f^*)} (\frac{1}{1- \gamma} + \frac{1}{2} ( 1 + \gamma ))
    \end{align}
    where the first inequality follows from Eq.~\eqref{eq:ineq_e_f_x_t_add_1_RDBD}, the second equality follows from simple algebra, the third inequality follows from Lemma~\ref{lem:bound_alpha}, the fourth equality follows from simple algebra, the fifth inequality follows from Lemma~\ref{lem:bound_alpha} and simple algebra, the sixth inequality follows from $\alpha_0 + T \eta \sigma^2 < 1$, the seventh equality follows from Definition~\ref{def:alpha_t}, the eighth inequality follows from $\langle g_t, g_{t-1} \rangle \leq \| g_t \|_2 \| g_t \|_2 \leq \sigma^2$, the ninth and the tenth inequalities follow from simple integral rules, the eleventh inequality follows from simple algebra, the twelfth equality follows from $\eta = \frac{\gamma\sqrt{f(x_0) - f^*}}{T\sigma^3\sqrt{L T}}$, the thirteenth inequality follows from simple algebra, the fourteenth equality follows from $\alpha_0 = \frac{\sqrt{f(x_0) - f^*}}{\sigma\sqrt{L T}}$, the last equality follows from simple algebra.

    Thus, we have
    \begin{align*}
        \min_t \{ \| \nabla f(x_t) \|_2^2 \} 
        = & ~ \sqrt{\min_t \{ \| \nabla f(x_t) \|_2^2 \}}  \\
        \leq & ~ \sqrt{\frac{1}{T} \sum_{t=1}^T \| \nabla f(x_t) \|_2^2}   \\
        \leq & ~ \sqrt{\frac{1}{T} \cdot \sigma \sqrt{LT(f(x_0) - f^*)} (\frac{1}{1- \gamma} + \frac{1}{2} ( 1 + \gamma )) } \\
        = & ~ \sqrt{\frac{1}{\sqrt{T}} \cdot \sigma \sqrt{L(f(x_0) - f^*)} (\frac{1}{1- \gamma} + \frac{1}{2} ( 1 + \gamma )) }
    \end{align*}
    where the first equality and the second inequality follow from simple algebras, the third inequality follows from Eq.~\eqref{eq:upper_bound_e_nabla_f}, the last equality follows from simple algebra.
\end{proof}

\section{Experiment Detail}\label{app:experiment}

In Appendix~\ref{sub:setup}, we provide our setup for our experiment. In Appendix~\ref{sub:result}, we provide more results of our experiments.

\subsection{Setup}\label{sub:setup}

\paragraph{Datasets. } We utilize the MNIST dataset \cite{lcb09} and the Cifar-10 dataset \cite{kh09} as our primary datasets for evaluating the effectiveness of our methods.

\paragraph{Models. } For the MNIST dataset, we employ a 3-layer linear network with the ReLU activation function \cite{nh10}. This architecture is appropriate for capturing the complex patterns and features present in the digit images of the MNIST dataset. On the other hand, for the Cifar-10 dataset, we utilize a 3-layer convolutional neural network (CNN) \cite{ksh12} with the ReLU activation function. The CNN architecture is specifically designed to effectively process and extract relevant features from the RGB images in the Cifar-10 dataset.

\paragraph{Metric.} We rely on the cross-entropy loss as the primary metric to evaluate the performance of our algorithms, that is $\mathcal{L}(\hat{y}, y) = - \frac{1}{N} \sum_{i=1}^N y_i \log \hat{y}_i$, where $N$ stands the number of classes. The cross-entropy loss measures the dissimilarity between the predicted probabilities and the true labels, providing valuable insights into the accuracy and effectiveness of our models.

\paragraph{Hyper-parameters. } For both the MNIST and Cifar-10 datasets, we establish the following hyperparameter settings:

\begin{itemize}
    \item Batch Size: We set the batch size to 16. This refers to the number of training examples processed in each iteration before updating the model's parameters.
    \item Initial Learning Rate: We initialize the learning rate, denoted as $\alpha_0$, to 0.005. The learning rate determines the step size at which the model adjusts its parameters during the training process.
    \item Learning Rate Schedule: We set $\eta$ to $0.01$ as the learning rate for the learning rate schedule. This schedule determines how the learning rate is adjusted during the training process based on certain criteria, such as the number of iterations or the validation performance.
\end{itemize}

By carefully selecting these hyper-parameters, we aim to strike a balance between model convergence and computational efficiency, ensuring an effective and stable training process for both the MNIST and Cifar-10 datasets.

\paragraph{Hyper-parameters for Adam+RDBD. } For our hybrid approach, which combines the Adam optimizer with our RDBD algorithm, we specify the following hyperparameter values:

\begin{itemize}
    \item Adam Hyperparameters: We set $\beta_1$ to 0.05 and $\beta_2$ to 0.99 as the hyperparameters for the Adam optimizer. These values control the exponential decay rates for the first and second moments of the gradients, respectively.
    \item Learning Rate Schedule: We set $\eta$ to 0.0000005 as the learning rate for the learning rate schedule. This schedule determines how the learning rate is adjusted during the training process based on certain criteria, such as the number of iterations or the validation performance.
    \item Upper bound on learning rate: Following Theorem~\ref{thm:main_result:formal}, we set a upper bound for $\alpha_t$ to prevent collapse phenomenon in training process.
\end{itemize}

By carefully selecting these hyperparameter values, we aim to leverage the strengths of both the Adam optimizer and our RDBD algorithm to achieve improved convergence and performance in our training process.

\paragraph{Platform and Device. } For our deep learning experiments, we utilize the PyTorch library \cite{pgm+19} as our chosen platform for training our models. PyTorch offers a comprehensive set of tools and functionalities that enable efficient and effective deep learning model development and training. To ensure optimal performance and accelerated computation, all of our experiments are conducted on a powerful RTX 3090 GPU device. The RTX 3090 GPU provides substantial computational capabilities, allowing us to leverage its parallel processing capabilities for faster and more efficient training of our deep learning models.

\subsection{Results}\label{sub:result}

\paragraph{Robustness on Different Initial Learning Rates}

To investigate the impact of the initial learning rate $\alpha_0$ on the training process, we conducted experiments using the RDBD algorithm on the MNIST dataset. We evaluated the loss reduction at the first 3750 steps for different initial learning rates, specifically $\alpha_0 \in \{ 0.01, 0.005, 0.001, 0.0005, 0.0001 \}$. The results, depicted in Figure~\ref{fig:lr_robustness}, demonstrate the ability of our RDBD algorithm to effectively speed up the training process and converge, even when the initial learning rate is as low as 0.0001.

Figure~\ref{fig:lr_robustness} presents a comparison of the loss values for different initial learning rates. It is evident that the RDBD algorithm exhibits robustness across a range of initial learning rates. Regardless of the specific value chosen, the algorithm is able to converge and achieve competitive performance within approximately 2500 steps.

This analysis highlights the robustness of the RDBD algorithm in the face of varying initial learning rates, underscoring its effectiveness for training neural networks on the MNIST dataset.

\begin{figure}[!ht]
    \centering
    \includegraphics[width=0.5\textwidth]{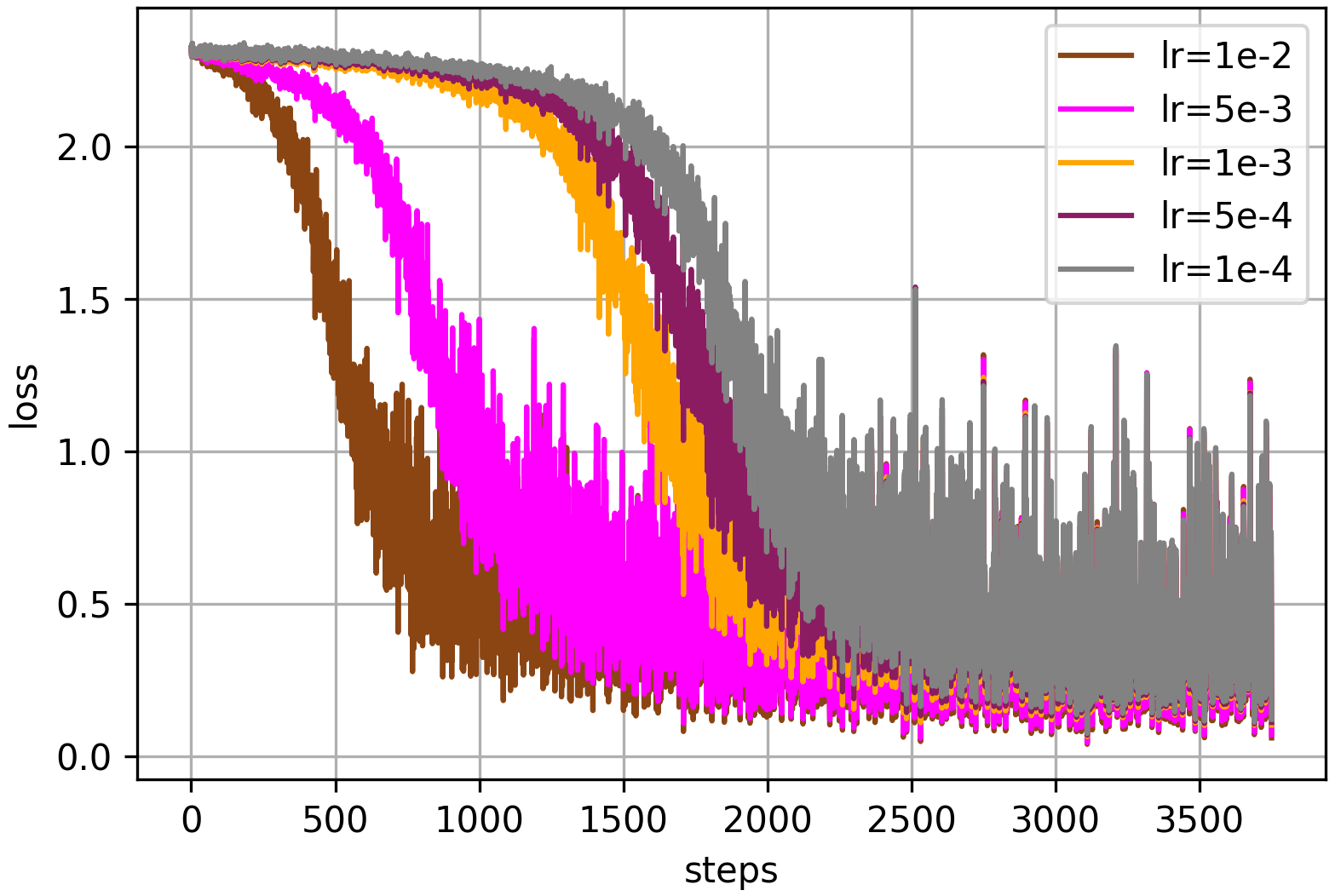}
    \caption{Comparison of loss of different initial learning rate. }
    \label{fig:lr_robustness}
\end{figure}

\begin{figure}[!ht]
    \centering
    \includegraphics[width=0.5\textwidth]{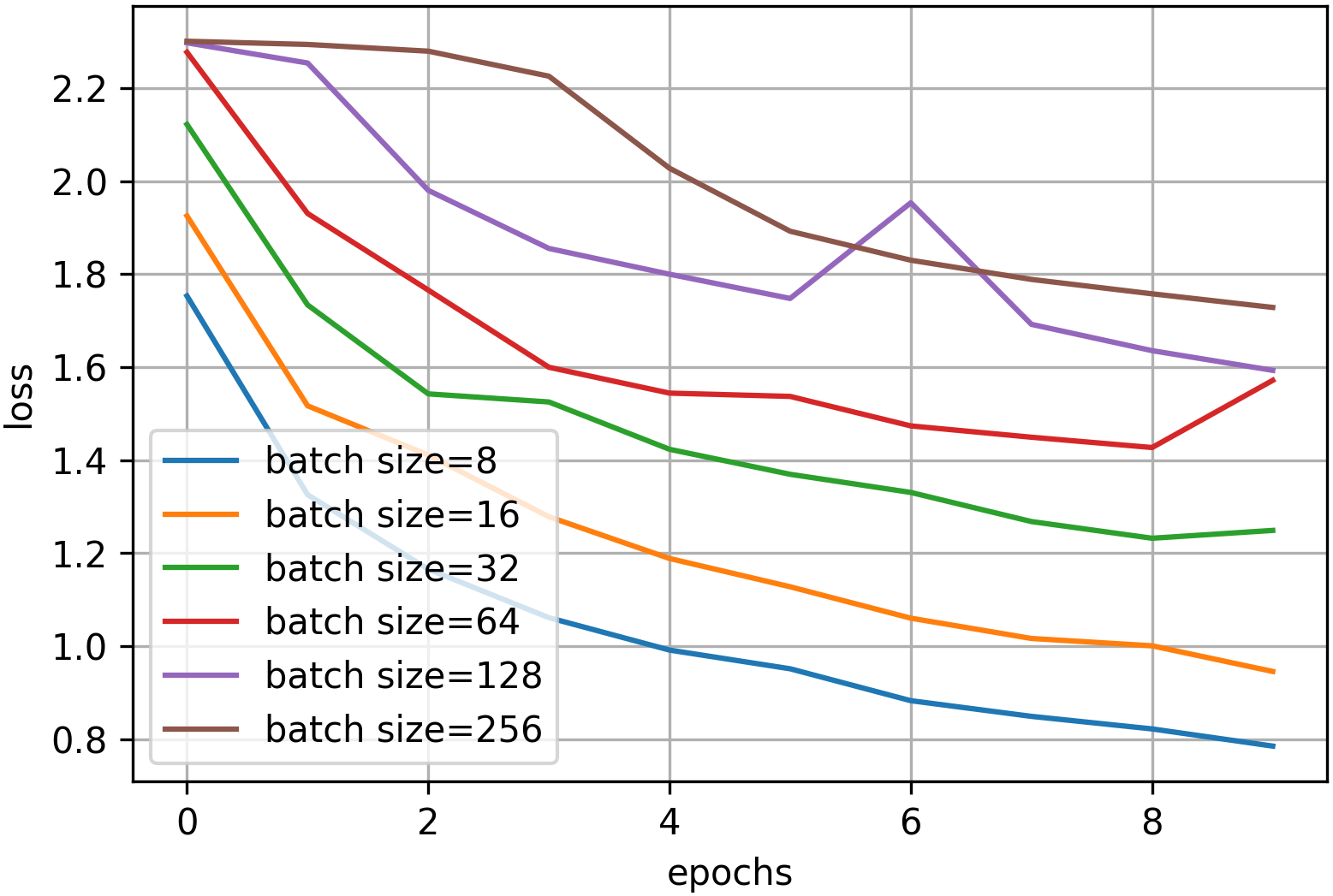}
    \caption{Comparison of loss of different batch sizes.  }
    \label{fig:bs_impact}
\end{figure}

\paragraph{Impact of Different Batch Sizes}

We investigated how batch size impacts convergence speed by running experiments on the CIFAR-10 dataset for the first 10 epochs using different batch sizes. As the results in Figure~\ref{fig:bs_impact} show, batch size has a significant effect on the convergence of the RDBD algorithm.

When the batch size is small, the RDBD algorithm accelerated quickly due to the faster gradient calculation and updates on each iteration. However, when the batch size is large, the RDBD algorithm contributed very little to speeding up convergence. This is likely because gradient calculations become inefficient for larger batch sizes, slowing down convergence.

In summary, small batch sizes allow for quicker convergence and acceleration of the RDBD algorithm during training on the CIFAR-10 dataset.

\ifdefined\isarxiv
\bibliographystyle{alpha}
\bibliography{ref}

\else
\fi




\end{document}